\documentclass{article} %
\usepackage{iclr2027_conference,times}
\usepackage[T1]{fontenc}
\usepackage[utf8]{inputenc}

\usepackage{amsmath,amsfonts,bm}

\def\eqref#1{equation~\ref{#1}}

\def\1{\bm{1}}

\DeclareMathAlphabet{\mathsfit}{\encodingdefault}{\sfdefault}{m}{sl}
\SetMathAlphabet{\mathsfit}{bold}{\encodingdefault}{\sfdefault}{bx}{n}

\usepackage[table]{xcolor}
\usepackage[pagebackref=true,breaklinks=true,colorlinks=true,bookmarks=false]{hyperref}
\definecolor{deepred}{HTML}{940000}
\hypersetup{linkcolor=deepred}

\definecolor{linkpurple}{rgb}{0.40,0.15,0.95}
\definecolor{lightblue}{rgb}{0.30,0.60,1.00}

\colorlet{citeblue}{linkpurple!60!lightblue} 
\colorlet{citemuted}{citeblue!90!gray} 

\hypersetup{
  citecolor = citemuted,
  urlcolor  = citemuted,
}

\usepackage{url}
\usepackage{amsmath}  %
\usepackage{amssymb}  %
\usepackage{booktabs} %
\usepackage{graphicx}
\usepackage{multirow}
\usepackage{tabularx}
\usepackage{enumitem}
\usepackage{wrapfig}
\usepackage{caption}
\usepackage{amsthm}
\newtheorem{lemma}{Lemma}
\newtheorem{proposition}{Proposition}
\newtheorem{corollary}{Corollary}

\usepackage{comment}

\usepackage[toc,page,header]{appendix}
\usepackage{minitoc}

\renewcommand \thepart{}
\renewcommand \partname{}

\usepackage{caption}

\definecolor{Gray}{gray}{0.92}
\newlength\savewidth\newcommand\shline{\noalign{\global\savewidth\arrayrulewidth
  \global\arrayrulewidth 1pt}\hline\noalign{\global\arrayrulewidth\savewidth}}

\title{\vspace{-4mm}\fontsize{16.5pt}{\baselineskip}\selectfont Rethinking Sequence Compression with\\Test-time Regression}

\author{%
  \fontsize{9pt}{\baselineskip}\selectfont
  Sihan Yang\textsuperscript{1}~~~Jie Huang\textsuperscript{2}~~~Hanxuan Li\textsuperscript{1}~~~Yifan Guo\textsuperscript{1}~~~Weigao Sun\textsuperscript{3}~~~Yiran Zhong\textsuperscript{3}~~~Weiyang Liu\textsuperscript{1,*}\\[0.175em]
  \fontsize{9pt}{\baselineskip}\selectfont
  \textsuperscript{1}The Chinese University of Hong Kong~~~\textsuperscript{2}Independent Researcher~~~\textsuperscript{3}Token Foundry, Alibaba Group\\[0.075em]
  \fontsize{9pt}{\baselineskip}\selectfont
  \textsuperscript{*}Corresponding author~~~~~~~~{\href{https://spherelab.ai/ttr-compressor}{\textbf{\texttt{spherelab.ai/ttr-compressor}}}}  
}

\newcommand{\ie}{\emph{i.e.}}
\newcommand{\eg}{\emph{e.g.}}

\iclrfinalcopy %
\begin{document}

\maketitle
\doparttoc
\faketableofcontents

\begin{figure}[ht]
    \vspace{-0.15in}
    \centering
    \includegraphics[width=.975\textwidth]{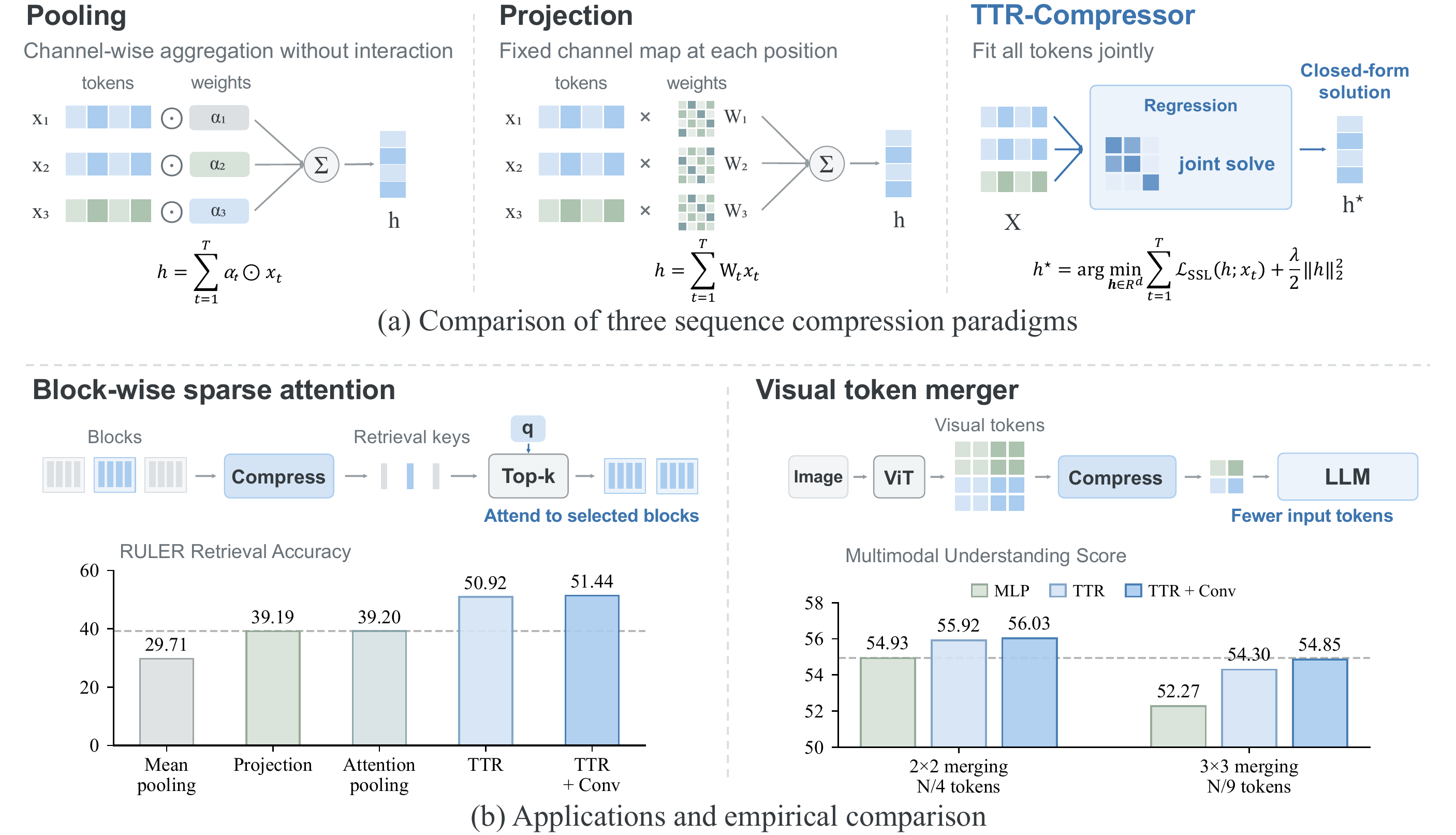}
    \vspace{-1.5mm}
    \caption{\footnotesize \textbf{Overview of TTR-Compressor.} (a) Unlike pooling and projection, which process tokens independently, TTR-Compressor fits all tokens jointly and thereby captures the interplay among them. (b) TTR-Compressor outperforms existing methods on both block-wise sparse attention and visual token merging.}
    \label{fig:teaser}
    \vspace{1.75mm}
\end{figure}

\begin{abstract}
\vspace{-1mm}
Sequence compression is a fundamental operation and a critical information bottleneck in efficient language and multimodal models, and yet prevailing approaches based on pooling or learned projections remain underperforming.
We revisit sequence compression through the lens of test-time regression: %
input tokens define a regression problem whose fitted coefficients form the compressed representation.
Importantly, the regression problem is learned end-to-end via backpropagation, while its coefficients are computed efficiently in closed form during the forward pass.
This leads to TTR-Compressor, a general alternative to conventional pooling and projection-based compression.
We evaluate TTR-Compressor as the compressor for block-wise sparse attention and the visual token merger for multimodal large language models. 
On block-wise sparse attention, TTR-Compressor improves RULER performance over attention-based compressors by a large margin.
For visual token merging, it consistently outperforms the standard MLP merger and matches its performance with 55.6\% fewer visual tokens.
These results establish test-time regression as an effective primitive for sequence compression in efficient language and multimodal models.

\end{abstract}

\vspace{-3mm}
\section{Introduction}
\label{sec:intro}
\vspace{-1.25mm}

Sequence compression is a ubiquitous operation in modern efficient large model architectures~\citep{qwen38flashnext,deepseekv41flash}, and yet remains surprisingly underexplored. For example, block-wise sparse attention compresses each token block into a summary vector used to score and select relevant blocks~\citep{minference,moba,nsa}, while multimodal large language models (MLLMs) merge neighboring visual features into compact tokens before feeding them into the language backbone~\citep{internvl35,qwen3vl,gemma4}. In both settings, the compressor acts as a critical information bottleneck. Once information is discarded, it cannot be recovered downstream.
One would expect compression to be a carefully designed architectural component. However, most models rely on simple mechanisms, such as parameter-free pooling, learned projections over concatenated tokens, or at best, attention-weighted pooling.

We argue that these default compressors share a fundamental structural limitation, \ie, \emph{they aggregate tokens using weights that do not account for inter-token relationships}.
Mean pooling uses uniform weights, linear projection uses fixed position-dependent weights, and attention-weighted pooling scores each token independently. 
As a result, redundant tokens can reinforce one another, while distinctive information is diluted.
From a statistical perspective, these methods primarily exploit \emph{first-order} token statistics while ignoring \emph{second-order} inter-token correlations, the very structure that reveals which information is redundant and which is distinctive.

We introduce a fundamentally different compression strategy. Rather than asking \emph{which weighted average} best summarizes a set of tokens, we let the tokens define a regression problem
and take its \emph{fitted coefficients} as the summary. Specifically, our proposed TTR-Compressor, constructs a self-supervised ridge-regression problem directly from the input tokens, requiring no external labels, and uses its closed-form solution as the compressed representation. 
This formulation creates a natural division: the outer training loop meta-learns \emph{which} regression problem produces task-useful representations, while the inner closed-form solution determines \emph{how} that problem should be fitted to each input. The key distinction from conventional pooling and projection arises from the Gram matrix in the regression solution. 
By explicitly capturing inter-token correlations, TTR-Compressor fits correlated tokens jointly, so redundant evidence has diminishing rather than purely additive influence. 
This input-adaptive behavior cannot be reproduced by a fixed projection or additive token combination.
Moreover, learning the regularizer removes the need for ridge hyperparameter tuning in large-scale pretraining. Because the compression window is typically smaller than the channel dimension (\ie, $T < d$), we solve the dual problem with an efficient $T \times T$ Cholesky factorization. 
Thus, TTR-Compressor incurs minimal overhead and introduces only $2d$ learnable parameters per compressor, making it an effective drop-in replacement for sequence compression.

We evaluate TTR-Compressor on two applications that impose fundamentally different requirements on compression. In \emph{block-wise sparse attention}, the compressed summary must act as a faithful \emph{retrieval key}. In this setting, TTR-Compressor consistently outperforms traditional baselines on RULER~\citep{ruler} retrieval tasks, surpassing sophisticated attention-based compressors~\citep{hils} by a considerable margin. In \emph{visual token merging}, by contrast, the compressed summary must preserve a faithful \emph{representation}. Here, TTR-Compressor consistently outperforms the widely used MLP merger across diverse compression ratios, and achieves comparable performance with 55.6\% fewer visual tokens. Our contributions are summarized below:
\begin{itemize}[leftmargin=*,itemsep=2pt,topsep=2pt]
  \item \textbf{A novel compression framework.} We introduce a test-time regression framework for sequence compression, reframing token summarization as an instance-adaptive optimization problem.
  \item \textbf{A practical operator.} We derive an efficient dual closed-form solution and eliminate hyperparameter tuning, making TTR-Compressor fully practical and scalable for large-scale pretraining.
  \item \textbf{Empirical effectiveness.} As a drop-in component in large models, TTR-Compressor consistently outperforms conventional baselines in both block-wise sparse attention and visual token merging.
\end{itemize}

\vspace{-1.5mm}
\section{Related Work}
\label{sec:related}
\vspace{-1.5mm}

\textbf{Sequence compression.}
Compressing a sequence of tokens into a single vector is a fundamental operation across modern architectures, with applications ranging from multimodal token reduction~\citep{tome,blip2,deco,honeybee,qwen2audio} to block-wise sparse attention~\citep{minference,moba,seer,flexprefill}.
Existing compression operators largely fall into three paradigms: 
(1) \emph{parameter-free pooling}, such as mean pooling; 
(2) \emph{concatenation followed by learned projection}; and 
(3) \emph{attention-based compression}, which summarizes tokens via attention-weighted averaging~\citep{landmark_attention,vflowopt}.
A fundamental limitation of these families is that they compute token combinations without explicitly accounting for the inter-token correlations. 
In contrast, TTR-Compressor derives its representation as the solution to a linear system whose Gram matrix directly encodes these correlations, providing an expressive, drop-in replacement for the operators above.

\textbf{Test-time regression and test-time training.}
The idea of adapting part of a network on each input at inference time dates back to fast-weight programmers~\citep{fast_weight_memory,fast_weight_programmers} and was revived as test-time training (TTT), first for robustness under distribution shift~\citep{ttt_shifts} and then as sequence-modeling layers whose hidden states are small models updated by gradient descent on a self-supervised loss~\citep{ttt_rnns,titans,nested_learning,vit3}. In-Place TTT~\citep{In-place_ttt} further extends this paradigm by directly adapting the existing MLP projection weights of pretrained Transformers. A parallel line of work reinterprets linear attention and its descendants (\eg, DeltaNet~\citep{deltanet}, Gated DeltaNet~\citep{gateddeltanet}, Longhorn~\citep{longhorn}, Titans~\citep{titans}) as online regression: each layer maintains a key-to-value regressor that is refined with one gradient or delta-rule step per token and read out by the query. MesaNet~\citep{mesanet} pushes the inner optimization to convergence, solving the problem at every time step with conjugate gradients, while LaCT~\citep{lact} updates fast weights over large, unordered chunks to improve hardware utilization.

TTR-Compressor differs from these methods in two fundamental ways. First, because compression operates independently on small, disjoint blocks without cross-block dependencies, the regression admits an efficient, \emph{exact closed-form solution} via a single $T\times T$ dual solve. This avoids both iterative gradient-based inner optimization (as in TTT, ViT$^3$, In-Place TTT, and MesaNet) and single-step approximations (such as delta-rule updates). Second, rather than using test-time regression to produce input-dependent model weights~\citep{meta_learning_with_closed-form_solvers,ttt_shifts} or a fast-weight memory that is subsequently queried~\citep{fast_weight_programmers,gateddeltanet,titans}, we use regression to directly yield a \emph{representation} of a token block. Specifically, the fitted coefficients themselves serve as the compressed representation (\eg, a visual token or a retrieval key).

\vspace{-1mm}
\section{Joint Sequence Compression via Test-time Regression}
\label{sec:methodology}
\vspace{-1mm}
We introduce {TTR-Compressor}, a sequence compression operator based on \textit{test-time regression}, as shown in Figure~\ref{fig:method}(a). Given a sequence of token representations, TTR-Compressor constructs a regression problem from the tokens and uses its fitted coefficients as the compressed representation. We first present the regression formulation of compression and specify its construction, then derive an efficient solution and describe how TTR-Compressor is learned end-to-end within a model.

\vspace{-.5mm}
\subsection{Compression as Test-time Regression}
\label{sec:ttr_framework}
\vspace{-.5mm}

We formulate sequence compression as ridge regression~\citep{ridge_regression}. The tokens in a sequence provide regression observations, and the fitted regression coefficients serve as a compact representation of that sequence. The regression objective encourages this representation to capture the relationships among tokens, while regularization controls the magnitude of the fitted coefficients.

Given a sequence $\bm{X}=[\bm{x}_1,\ldots,\bm{x}_T]^\top
\in\mathbb{R}^{T\times d}$, our goal is to compress its $T$ token representations into a single vector $\bm{h}\in\mathbb{R}^{d}$.
We associate each token with a regression key $\bm{k}_t\in\mathbb{R}^{d}$ and a scalar target $v_t\in\mathbb{R}$, and define the compressed representation as:
\begin{equation}
    \bm{h}^{\star}
    =
    \arg\min_{\bm{h}\in\mathbb{R}^{d}}
    \frac{1}{2}\sum_{t=1}^{T}
    \omega_t\big(\bm{k}_t^\top\bm{h}-v_t\big)^2
    +
    \frac{\lambda}{2}\|\bm{h}\|_2^2,
    \label{eq:ttr_objective}
\end{equation}
where $\omega_t\geq 0$ determines the contribution of each token and $\lambda>0$ controls regularization.
$\bm{h}^{\star}$ summarizes the sequence by fitting the key-target relationships across its tokens. 
The regression target $v_t$ is \emph{meta-learned} through the task loss, while the regression itself is constructed and solved for each input sequence during both training and inference.
Overall, we are solving a bi-level optimization, where the outer objective is the task loss that finds the task-useful target $v_t$, and the inner objective is the test-time regression Eq.~\ref{eq:ttr_objective} that solves the compressed representation $\bm{h}^{\star}$. Here, \textit{test-time} refers to fitting these input-specific coefficients; the learned model parameters remain fixed at inference.

\begin{figure}[t]
    \vspace{-0.15in}
    \centering
    \includegraphics[width=\textwidth]{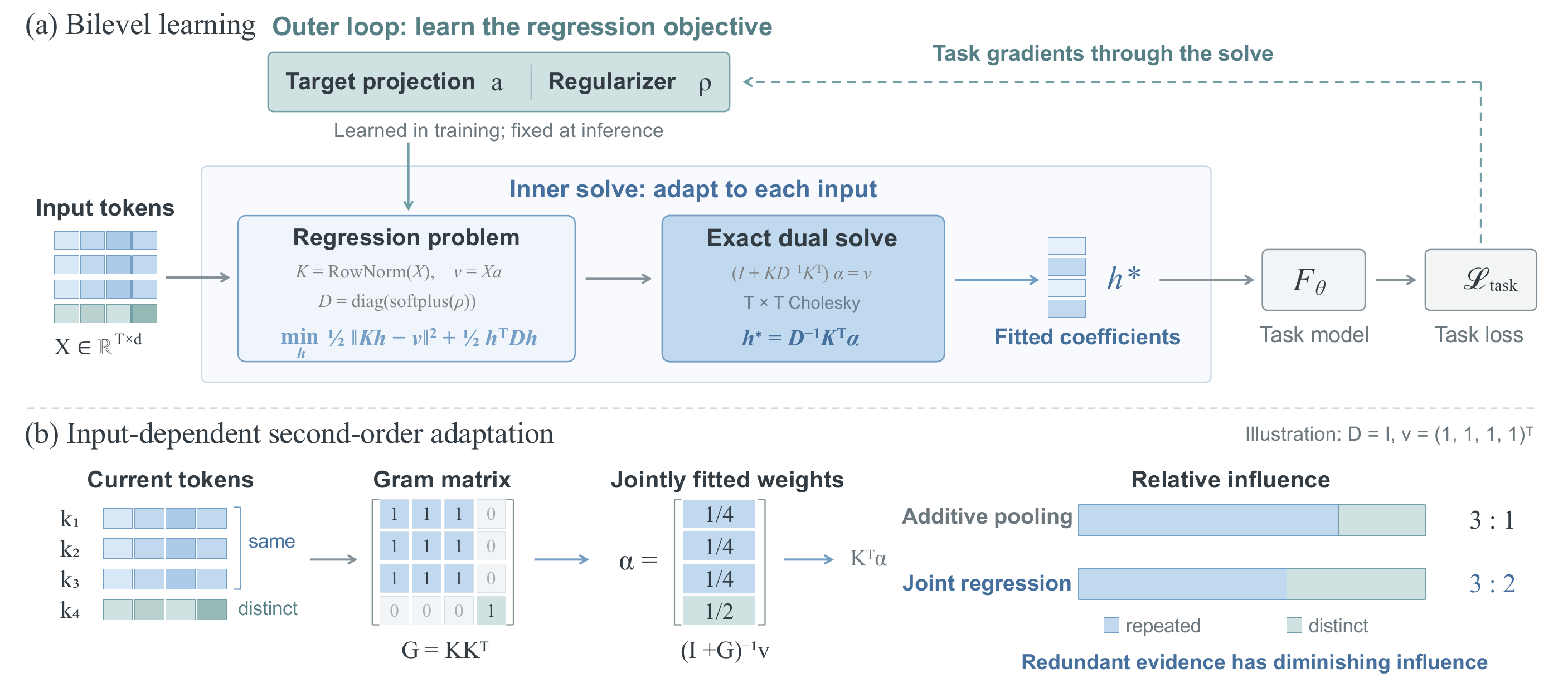}
    \vspace{-0.23in}
    \caption{\footnotesize \textbf{Overview of TTR-Compressor.} (a) Bilevel learning: The inner solve formulates sequence compression as test-time ridge regression with an exact dual Cholesky solve to obtain the representation $\bm{h}^*$, while the regression parameters are learned end-to-end via task gradients. (b) Input-dependent second-order adaptation: By jointly coupling tokens through the Gram matrix, TTR naturally diminishes the relative influence of redundant tokens compared to additive pooling. For intuitive illustration, we set $\bm{D}=\bm{I}$ and uniform targets $\bm{v}=\mathbf{1}$; as typically learned in practice, a smaller regularizer $\bm{D}$ further amplifies this deduplication effect, yielding an even higher relative share for distinct tokens.}
    \vspace{-0.06in}
\label{fig:method}
\end{figure}

\vspace{-.5mm}
\subsection{Constructing the Regression Problem}
\label{sec:ttr_construction}
\vspace{-.5mm}

\textbf{Keys and Targets.}
We construct regression keys by normalizing the input features and generate scalar targets using a learned projection:
$\bm{k}_t=\frac{\bm{x}_t}{\|\bm{x}_t\|_2},v_t = \bm{a}^\top\bm{x}_t,$
where $\bm{a}\in\mathbb{R}^{d}$ is shared across tokens and sequences within a compressor.
Key normalization ensures that $\bm{K}^\top\bm{K}$ measures directional (semantic) similarity alone~\citep{liu2017deep}. Crucially, this normalization is asymmetric, as targets retain original token magnitudes ($v_t = \bm{a}^\top \bm{x}_t$). If targets were also normalized ($v_t = \bm{a}^\top \bm{k}_t$), the fitting term would reduce to $\frac{1}{2}\|\bm{K}(\bm{h}-\bm{a})\|_2^2$, trivially collapsing $\bm{h}^\star$ toward the sequence-invariant static vector $\bm{a}$. By preserving token norms in the targets, no constant vector can fit varying magnitudes, forcing $\bm{h}^{\star}$ to dynamically integrate both token semantics and their relative importance.
The targets are derived entirely from the input, so fitting $\bm{h}^{\star}$ requires no additional labels. 
Optionally, targets can be further enriched by first applying a local convolution (1D for language and 2D for vision) to aggregate neighboring context before the projection, \textit{i.e.}, $v_t = \bm{a}^\top \operatorname{Conv}(\bm{X})$, where $\operatorname{Conv}(\cdot)$ denotes the convolution operator and $\bm{X}$ represents the local context window of $n$ neighboring tokens around position $t$. Alternative parameterizations of $v_t$ are explored in the Appendix~\ref{sec:v_ablation}.

\textbf{Uniform Token Contributions.}
We use uniform weights $\omega_t=1$, assigning the same coefficient to every token's regression error.
Let $\bm{K}=[\bm{k}_1,\ldots,\bm{k}_T]^\top\in\mathbb{R}^{T\times d}$ and $\bm{v}=[v_1,\ldots,v_T]^\top\in\mathbb{R}^{T}$.
Setting the gradient of Eq.~\ref{eq:ttr_objective} to zero gives us a closed-form solution:
\begin{equation}
    (\bm{K}^\top\bm{K}+\lambda\bm{I}_d)\bm{h}^{\star}
    = \bm{K}^\top\bm{v},
    \qquad
    \bm{h}^{\star}
    =
    (\bm{K}^\top\bm{K}+\lambda\bm{I}_d)^{-1}
    \bm{K}^\top\bm{v}.
    \label{eq:ttr_primal}
\end{equation}
Since $\lambda>0$, the objective is strictly convex and has a unique solution.
The solution combines the key-target statistic $\bm{K}^\top\bm{v}$ with the sequence-dependent feature geometry $\bm{K}^\top\bm{K}$.

\textbf{Learnable Diagonal Regularization.}
\label{sec:ttr_diagonal}
In conventional ridge regression, $\lambda$ is a tunable hyperparameter. Such tuning is expensive during large-model pretraining and complicates scaling to different model sizes. To avoid this, we replace $\lambda\bm{I}_d$ with a learnable positive diagonal matrix (meta-learned via the task loss), yielding a generalized ridge regression objective~\citep{tikhonov_solutions}:
\begin{equation}
    \bm{D}
    =
    \operatorname{diag}(\delta_1,\ldots,\delta_d),
    \qquad
    \delta_j
    =
    \operatorname{softplus}(\rho_j),
    \label{eq:ttr_diagonal}
\end{equation}
where $\bm{\rho}\in\mathbb{R}^{d}$ is trainable.
The resulting objective and its solution are:
\begin{align}
    \bm{h}^{\star}
    &=
    \arg\min_{\bm{h}}
    \frac{1}{2}\|\bm{K}\bm{h}-\bm{v}\|_2^2
    +
    \frac{1}{2}\bm{h}^\top\bm{D}\bm{h},
    \label{eq:ttr_diagonal_objective}\\
    &=
    (\bm{K}^\top\bm{K}+\bm{D})^{-1}
    \bm{K}^\top\bm{v},
    \label{eq:ttr_diagonal_primal}
\end{align}
which solves a regularized linear system that takes the pairwise interaction of tokens into account.

\vspace{-1mm}
\subsection{Efficient Dual-form Solution}
\label{sec:ttr_dual}
\vspace{-.5mm}

Directly evaluating Eq.~\ref{eq:ttr_primal} requires solving a $d\times d$ linear system.
When the number of tokens being compressed is typically smaller than the feature dimension, we instead use the equivalent dual form~\citep{ridge_dual} in order to improve the computational efficiency:
\begin{equation}
    \bm{h}^{\star}
    =
    \bm{K}^\top
    (\bm{K}\bm{K}^\top+\lambda\bm{I}_T)^{-1}
    \bm{v}.
    \label{eq:ttr_dual}
\end{equation}
This replaces the feature-space system with a $T\times T$ token-space system.
Its entries depend on pairwise similarities between regression keys, allowing the fitted representation to account for correlations among tokens.
We evaluate Eq.~\ref{eq:ttr_dual} by solving
$(\bm{K}\bm{K}^\top+\lambda\bm{I}_T)\bm{\alpha}=\bm{v}$
and computing $\bm{h}^{\star}=\bm{K}^\top\bm{\alpha}$, without explicitly forming an inverse.
The resulting computational cost is
$\mathcal{O}(T^2d+T^3)$, compared with
$\mathcal{O}(Td^2+d^3)$ for the direct primal formulation.
The dual form is therefore particularly suitable for compressing short sequences of high-dimensional features.

\textbf{Obtaining an Efficient Dual Solution.}
With the learnable diagonal regularization introduced in Section~\ref{sec:ttr_construction}, $\lambda\bm{I}_d$ is replaced by $\bm{D}$. To obtain an efficient dual solution, we rescale the keys as $\widetilde{\bm{K}}=\bm{K}\bm{D}^{-1/2}$, yielding the final form of TTR-Compressor:
\begin{equation}
    \boxed{
    \operatorname{TTR}(\bm{X})
    =
    \bm{D}^{-1/2}\widetilde{\bm{K}}^\top
    \big(
        \bm{I}_T+
        \widetilde{\bm{K}}\widetilde{\bm{K}}^\top
    \big)^{-1}
    \bm{v}
    }
    \label{eq:ttr_final}
\end{equation}
Because $\bm{D}$ is diagonal, rescaling the keys costs only $\mathcal{O}(Td)$, preserving the dual formulation's overall complexity.
In practice, we compute a Cholesky factorization
$\bm{L}\bm{L}^\top
=\bm{I}_T+\widetilde{\bm{K}}\widetilde{\bm{K}}^\top$,
solve $\bm{L}\bm{L}^\top\bm{\alpha}=\bm{v}$ through triangular solutions, and return
$\bm{D}^{-1/2}\widetilde{\bm{K}}^\top\bm{\alpha}$.

\vspace{-.5mm}
\subsection{Fully End-to-end Training}
\label{sec:ttr_learning}
\vspace{-.5mm}

The target projection $\bm{a}$ and diagonal parameters $\bm{\rho}$ are optimized end-to-end through the differentiable closed-form solution using the original model's training objective.
Each compressor maintains its own parameters, which are shared across all input sequences. At inference time, these learned parameters remain fixed, while $\bm{h}^{\star}$ is recomputed independently for each input sequence. 
TTR-Compressor therefore combines a learned regression objective with an exact, input-adaptive solution, without requiring iterative gradient-based optimization at test time.

\vspace{-1mm}
\section{Understanding TTR-Compressor}
\label{sec:understanding_ttr}
\vspace{-1mm}

The learning process of TTR-Compressor can be decomposed into two distinct mechanisms: outer-loop training learns which regression problem is useful for compression, while the inner solution adapts the compressed representation to the current input's statistics. We examine them separately.

\vspace{-.5mm}
\subsection{Task-Supervised Learning of the Compression Objective}
\label{sec:ttr_outer_learning}
\vspace{-.5mm}

The regression objective is not prescribed independently of the task. Its targets $v_t=\bm{a}^\top\bm{x}_t$ and regularization matrix $\bm{D}$ are learned through the downstream use of the compressed representation. Writing $\bm{h}^{\star}_{\bm{a},\bm{\rho}}$ for the regression solution, the training objective takes the form:
\begin{equation}
    \min_{\theta,\bm{a},\bm{\rho}}
    \mathbb{E}_{(x,y)}
    \left[
        \mathcal{L}_{\mathrm{task}}
        \left(
            F_{\theta}\!\left(
                \bm{h}^{\star}_{\bm{a},\bm{\rho}}
                (\bm{X}_{\theta}(x))
            \right),
            y
        \right)
    \right],
    \label{eq:ttr_outer_objective}
\end{equation}
where $\bm{X}_{\theta}$ produces the input representations and $F_{\theta}$ denotes the subsequent computation. Thus, the outer objective evaluates the usefulness of $\bm{h}^{\star}$ for the final task, rather than leveraging its learnability to merely minimize the inner regression error.

\textbf{Learning the Targets and Regularization.}
This coupling can be seen directly from the gradients.
Let
$\bm{A}=\bm{K}^{\top}\bm{K}+\bm{D}$,
$\bm{g}=\nabla_{\bm{h}^{\star}}\mathcal{L}_{\mathrm{task}}$,
and $\bm{s}=\bm{A}^{-1}\bm{g}$.
Differentiating the regression solution gives:
\begin{equation}
    \nabla_{\bm{a}}\mathcal{L}_{\mathrm{task}}
    =
    \bm{X}^{\top}\bm{K}\bm{s},
    \qquad
    \nabla_{\bm{\rho}}\mathcal{L}_{\mathrm{task}}
    =
    -\operatorname{sigmoid}(\bm{\rho})
    \odot\bm{h}^{\star}\odot\bm{s},
    \label{eq:ttr_outer_gradients}
\end{equation}
where $\odot$ denotes element-wise multiplication. The downstream gradient therefore shapes both the scalar targets derived from tokens via $\bm{a}$ and the coordinate-wise penalties imposed through $\bm{\rho}$. Consequently, the scalar targets require no predefined semantics; instead, they serve as learned measurements whose fitted coefficients yield informative representations for downstream layers. In this sense, the regression provides a structural inductive bias for compression, while outer-loop optimization steers this structural bias towards better handling the downstream task.

\vspace{-.5mm}
\subsection{Input-Dependent Second-Order Adaptation}
\label{sec:ttr_spectral}
\vspace{-.5mm}

Beyond learning the regression objective, the inner solve adapts the
compressed representation to the correlation structure of each input.
For isotropic regularization, the dual-form solution is
\begin{equation}
    \bm h^\star
    =
    \bm K^\top
    (\bm K \bm K^\top + \lambda \bm I)^{-1}
    \bm v.
\end{equation}
Let
    $\bm G = \bm K\bm K^\top \in \mathbb{R}^{T\times T}.$
Since the keys are normalized, each entry
$G_{ij}=\bm k_i^\top\bm k_j$ directly measures the similarity between
tokens $i$ and $j$.
Thus, $\bm G$ captures the input-dependent second-order correlation
structure of the token group.

The compressed representation can equivalently be written as
\begin{equation}
    \bm h^\star = \bm K^\top \bm\alpha,
    \qquad
    \bm\alpha=(\bm G+\lambda\bm I)^{-1}\bm v.
\end{equation}
Unlike pooling and projection, where token contributions are aggregated independently,
the effective weights $\bm\alpha$ are determined jointly by the target
responses and all pairwise token correlations.
Highly similar tokens are explicitly coupled through the Gram matrix,
so redundant evidence exhibits diminishing rather than additive
influence.
Conversely, information not already explained by strongly
correlated tokens can retain a larger relative influence when supported
by the regression targets, as shown in Figure~\ref{fig:method}(b).
TTR therefore adapts its compression rule to the second-order structure
of every input rather than applying a fixed or independently weighted
combination of tokens.

\textbf{Learning the inter-token geometry.}
With the learnable diagonal regularizer $\bm D$, this interpretation
extends to the rescaled keys
$\widetilde{\bm K}=\bm K\bm D^{-1/2}$.
The corresponding token Gram matrix becomes
\begin{equation}
    \widetilde{\bm G}
    =
    \widetilde{\bm K}\widetilde{\bm K}^\top
    =
    \bm K\bm D^{-1}\bm K^\top.
\end{equation}
Training therefore learns which feature dimensions should matter more
when measuring token similarity.
The outer loop learns a task-adapted correlation geometry, while the
inner solve adapts the compressed representation to the correlation
structure of each individual input.

\vspace{-1mm}
\section{Applications and Experimental Results}
\label{sec:experiments}
\vspace{-1mm}

We evaluate TTR-Compressor in two settings: block-wise sparse attention for efficient language modeling, where compressed representations guide context selection, and visual token merging for efficient multimodal understanding, where they serve as the visual input to the language backbone. 

\vspace{-.5mm}
\subsection{Block-wise Sparse Attention}
\label{sec:exp_sparse}
\vspace{-.5mm}

\textbf{Architecture and Baselines.}
Block-wise sparse attention partitions the context into blocks, scores their compressed representations against each query, and attends to tokens in selected blocks alongside a local window. The compressor therefore determines which distant information remains accessible: an uninformative summary can prevent a relevant block from being selected, regardless of the quality of its individual token representations.

We adopt HiLS-Attention~\citep{hils} as a base framework and vary only the compressor. HiLS-Attention has three effective architectural components: hierarchical softmax attention, which combines within-block attention with learned block-level weights; low-rank query calibration, which adapts queries for scoring compressed representations; and HoPE, which applies RoPE to only a subset of feature dimensions to support length extrapolation. We compare four compressors: mean pooling, concat \& projection, the attention-weighted pooling used in the original HiLS-Attention, and TTR-Compressor. Full attention and sliding-window attention serve as additional references.

\textbf{Training and Evaluation.}
We train 345M-parameter language models on the FineWeb dataset~\citep{fineweb} with an 8K context length. Training runs for 30,000 optimization steps with a global batch size of 128 sequences, corresponding to approximately 30B tokens, and uses AdamW with a learning rate of $3\times10^{-4}$. All sparse attention variants use 64-token blocks, select up to 32 distant blocks per query, and retain a 512-token local window.

Evaluation covers two complementary aspects. First, since queries use the compressed representations to retrieve relevant blocks, retrieval accuracy directly reflects compression quality. Following~\citep{GDN2}, we evaluate on the single-needle (S-NIAH) and multi-key (MK-NIAH) retrieval tasks from RULER~\citep{ruler}, which require locating relevant information among distractors, as the most direct and demanding test of the compressor.  We evaluate at 4K--32K context lengths, including extrapolation beyond the 8K training length. Additionally, LAMBADA accuracy and the zero-shot benchmarks in Table~\ref{tab:commonsense} assess contextual prediction, commonsense reasoning, and question answering.

\begin{table*}[t]
\centering
\scriptsize
\renewcommand{\arraystretch}{1.26}
\setlength{\tabcolsep}{3pt}
\setlength{\abovecaptionskip}{5pt}
\setlength{\belowcaptionskip}{3pt}
\newcommand{\gr}[1]{\textcolor{gray}{#1}}
\resizebox{\linewidth}{!}{%
\begin{tabular}{lccccccccccccccccc}
\multirow{2}{*}{\textbf{Model}} & \multicolumn{4}{c}{\textbf{S-NIAH-1}} & \multicolumn{4}{c}{\textbf{S-NIAH-2}} & \multicolumn{4}{c}{\textbf{S-NIAH-3}} & \multicolumn{4}{c}{\textbf{MK-NIAH-1}} & \multirow{2}{*}{\textbf{Avg.}}\\
 & 4K & 8K & 16K & 32K & 4K & 8K & 16K & 32K & 4K & 8K & 16K & 32K & 4K & 8K & 16K & 32K &\\
\shline
 \gr{Full Attention}
 & \gr{0.834} & \gr{0.428} & \gr{0} & \gr{0} & \gr{0.850} & \gr{0.412} & \gr{0} & \gr{0} & \gr{0.572} & \gr{0.470} & \gr{0} & \gr{0} & \gr{0.374} & \gr{0.272} & \gr{0} & \gr{0} & \gr{0.2632}\\
 \gr{Sliding Window Attention}
 & \gr{0.116} & \gr{0.050} & \gr{0.044} & \gr{0.016} & \gr{0.148} & \gr{0.048} & \gr{0.034} & \gr{0.030} & \gr{0.112} & \gr{0.094} & \gr{0.030} & \gr{0.034} & \gr{0.126} & \gr{0.080} & \gr{0.026} & \gr{0.028} & \gr{0.0635}\\
 Mean Pooling
 & 0.938 & 0.536 & 0.470 & 0.270 & 0.840 & 0.258 & 0.042 & 0.050 & 0.292 & 0.282 & 0.040 & 0.042 & 0.366 & 0.230 & 0.042 & 0.056 & 0.2971\\
 Concat \& Projection
 & 0.892 & 0.840 & 0.792 & 0.582 & 0.764 & 0.232 & 0.124 & 0.068 & 0.512 & 0.324 & \underline{0.174} & \underline{0.076} & 0.476 & 0.278 & 0.084 & 0.052 & 0.3919\\
 Attention-Weighted Pooling
 & 0.944 & 0.914 & 0.876 & \underline{0.740} & 0.926 & 0.274 & \underline{0.136} & 0.060 & 0.340 & 0.082 & 0.026 & 0.026 & 0.474 & 0.258 & \underline{0.124} & 0.072 & 0.3920\\
\rowcolor{Gray}
 TTR-Compressor
 & \underline{0.996} & \underline{0.996} & \underline{0.992} & 0.650 & \textbf{0.982} & \textbf{0.566} & \textbf{0.288} & \textbf{0.126} & \underline{0.770} & \underline{0.376} & \textbf{0.188} & 0.066 & \underline{0.528} & \textbf{0.302} & \textbf{0.212} & \textbf{0.110} & \underline{0.5092}\\
\rowcolor{Gray}
 TTR-Compressor (w/ Conv)
 & \textbf{0.998} & \textbf{1.000} & \textbf{0.994} & \textbf{0.988} & \underline{0.960} & \underline{0.454} & 0.128 & \underline{0.102} & \textbf{0.834} & \textbf{0.482} & 0.130 & \textbf{0.096} & \textbf{0.542} & \underline{0.298} & 0.116 & \underline{0.108} & \textbf{0.5144}\\
\end{tabular}
}
\caption{\footnotesize Accuracy on Single Needle-In-A-Haystack (S-NIAH) and Multi-Key Needle-In-A-Haystack (MK-NIAH) tasks from RULER across varied context lengths. Best values for each context length are \textbf{bolded}, and second-best values are \underline{underlined}.}
\label{tab:ruler}
\vspace{-.5mm}
\end{table*}

\begin{table*}[t]
\centering
\scriptsize
\renewcommand{\arraystretch}{1.23}
\setlength{\tabcolsep}{6pt}
\setlength{\abovecaptionskip}{5pt}
\setlength{\belowcaptionskip}{3pt}
\newcommand{\gr}[1]{\textcolor{gray}{#1}}
\resizebox{\linewidth}{!}{%
\begin{tabular}{lccccccccccc}
\multirow{2}{*}{\textbf{Model}} & \textbf{ARC-c} & \textbf{ARC-e} & \textbf{BoolQ} & \textbf{COPA} & \textbf{Hella.} & \textbf{LMB.} & \textbf{OBQA} & \textbf{PIQA} & \textbf{SciQ} & \textbf{Wino.} & \textbf{Avg.}\\
 & acc $\uparrow$ & acc $\uparrow$ & acc $\uparrow$ & acc $\uparrow$ & acc $\uparrow$ & acc $\uparrow$ & acc $\uparrow$ & acc $\uparrow$ & acc $\uparrow$ & acc $\uparrow$ & acc $\uparrow$\\
\shline
 \gr{Full Attention}
 & \gr{24.74} & \gr{42.09} & \gr{60.15} & \gr{68.00} & \gr{39.85} & \gr{35.71} & \gr{29.60} & \gr{67.74} & \gr{62.70} & \gr{51.07} & \gr{48.16}\\
 \gr{Sliding Window Attention}
 & \gr{24.91} & \gr{41.96} & \gr{55.02} & \gr{69.00} & \gr{39.53} & \gr{35.47} & \gr{29.00} & \gr{67.14} & \gr{63.00} & \gr{50.75} & \gr{47.58}\\
 Mean Pooling
 & 24.66 & \textbf{42.38} & 55.99 & 66.00 & \underline{39.62} & 35.88 & 28.40 & \textbf{69.31} & \textbf{65.50} & \textbf{53.04} & 48.08\\
 Concat \& Projection
 & \underline{24.83} & \underline{42.30} & 58.65 & 68.00 & 39.58 & \textbf{36.62} & \underline{30.00} & 67.74 & 62.20 & \underline{51.78} & \textbf{48.17}\\
 Attention-Weighted Pooling
 & \textbf{25.00} & 41.75 & \textbf{61.22} & 65.00 & 39.61 & 35.61 & 29.40 & \underline{68.93} & \underline{63.80} & 50.51 & 48.08\\
\rowcolor{Gray}
 TTR-Compressor
 & 23.29 & 41.58 & 57.95 & \textbf{71.00} & 39.41 & \underline{36.27} & \textbf{30.20} & \textbf{69.31} & 63.00 & 49.57 & \underline{48.16}\\
\rowcolor{Gray}
 TTR-Compressor (w/ Conv)
 & 24.15 & 41.33 & \underline{59.94} & \underline{70.00} & \textbf{39.63} & 35.36 & \textbf{30.20} & 67.85 & 62.80 & 49.96 & 48.12\\
\end{tabular}
}
\caption{\footnotesize Performance comparison on language modeling and zero-shot common-sense reasoning. The evaluation context lengths are too short to activate sparse attention, so the compressor is not involved during inference; these results therefore only reflect the common-sense reasoning ability acquired during language model pretraining. Note that Concat \& Projection has significantly more parameters than the others. Best values are \textbf{bolded}, and second-best values are \underline{underlined}.}
\label{tab:commonsense}
\vspace{-2.5mm}
\end{table*}

\textbf{Results.} As shown in Table~\ref{tab:ruler}, the choice of compressor largely determines retrieval quality. TTR-Compressor outperforms all baselines by a large margin, and its advantage persists when extrapolating to 16K and 32K. We attribute this to its second-order modeling: pooling and projection add up tokens independently, so the redundant context that fills most of a block drowns out the few informative tokens. The Gram matrix in TTR-Compressor fits correlated tokens jointly and reduces the influence of redundant evidence, so the resulting key keeps the distinctive content that queries need to find. The convolutional variant brings a further gain, reaching 12.24 points over the strongest baseline, as convolution makes the regression targets more informative and thus improves compression quality. On the benchmarks in Table~\ref{tab:commonsense}, the evaluation contexts are too short to activate sparse attention, so these results only reflect the common-sense reasoning ability acquired during pretraining. Concat \& Projection has substantially more parameters than the other compressors, which may benefit pretraining. TTR-Compressor performs well and on par with full attention.

\vspace{-.5mm}
\subsection{Visual Token Merger}
\label{sec:exp_visual}
\vspace{-.5mm}

\textbf{Architecture and Baselines.}
High-resolution images produce many visual tokens, making their direct processing by a language backbone expensive. Modern multimodal models therefore commonly use a spatial merger to reduce the visual sequence length. A widely adopted design concatenates each $2\times2$ group of neighboring visual features and maps it to one token through a two-layer MLP~\citep{internvl35,qwen3vl,qwen35,kimik3}. Longer multimodal contexts and agentic workflows with repeated visual observations further motivate stronger compression, such as $3\times3$ merging~\citep{gemma4,deepseekv41flash}.

We compare TTR-Compressor with the standard MLP merger. TTR-Compressor fits one regression per spatial group to compress its four or nine visual features into a single vector, which a two-layer MLP then projects to the language model's hidden dimension. Its convolutional variant applies a $3\times3$ depthwise convolution to the feature map before constructing the regression targets.

\textbf{Training and Evaluation.}
We use the ViT of Qwen3-VL~\citep{qwen3vl} as the vision encoder and Qwen3-4B-Instruct~\citep{qwen3} as the language model. Training consists of two stages. The first stage trains only the merger modules on LLaVA-558K~\citep{improvedllava} for one epoch with a learning rate of $10^{-3}$. The second stage jointly trains the merger modules and language backbone on 1.3M instruction samples mixing image~\citep{llavanext}, multi-image~\citep{llavanext-interleave}, and video~\citep{llavavideo} data for one epoch with a learning rate of $2\times10^{-5}$.

We evaluate on 4 types of benchmarks: general visual question answering (MMBench~\citep{mmbench}, MMStar~\citep{mmstar}, RealWorldQA~\citep{grok15v}), knowledge-intensive multimodal reasoning (MMMU~\citep{mmmu}, AI2D~\citep{ai2d}), text, infographic, and chart understanding (OCRBench~\citep{ocrbench}, InfoVQA~\citep{infovqa}, ChartQA~\citep{chartqa}), and spatial reasoning (MMSI~\citep{mmsi}, ERQA~\citep{erqa}).

\begin{table*}[t]
\centering
\scriptsize
\renewcommand{\arraystretch}{1.23}
\setlength{\tabcolsep}{3pt}
\setlength{\abovecaptionskip}{5pt}
\setlength{\belowcaptionskip}{3pt}
\resizebox{\linewidth}{!}{%
\begin{tabular}{lccccccccccc}
\multirow{2}{*}{\textbf{Method}}
& \multicolumn{3}{c}{\textbf{General VQA}}
& \multicolumn{2}{c}{\textbf{Knowledge Reasoning}}
& \multicolumn{3}{c}{\textbf{Document Understanding}}
& \multicolumn{2}{c}{\textbf{Spatial Under.}}
& \multirow{2}{*}{\textbf{Avg.}} \\
& \textbf{MMBench}
& \textbf{MMStar}
& \textbf{RWQA}
& \textbf{MMMU}
& \textbf{AI2D}
& \textbf{OCRBench}
& \textbf{InfoVQA}
& \textbf{ChartQA}
& \textbf{MMSI}
& \textbf{ERQA}
& \\
\shline

\multicolumn{12}{l}{\textit{$2\times2$ Token Compression}} \\

MLP
& \underline{76.63}
& 49.93
& \underline{64.44}
& \textbf{46.22}
& 75.49
& 62.80
& 41.72
& 70.64
& 26.20
& 35.25
& 54.93 \\

\rowcolor{Gray}
TTR-Compressor
& \textbf{76.89}
& \textbf{53.94}
& 63.40
& 45.44
& \textbf{76.94}
& \underline{62.90}
& \underline{43.74}
& \textbf{73.04}
& \textbf{26.90}
& \textbf{36.00}
& \underline{55.92} \\

\rowcolor{Gray}
TTR-Compressor (w/ Conv)
& 75.86
& \underline{52.26}
& \textbf{66.93}
& \underline{45.67}
& \underline{76.65}
& \textbf{63.80}
& \textbf{44.33}
& \underline{72.52}
& \underline{26.30}
& \textbf{36.00}
& \textbf{56.03} \\

\hline
\multicolumn{12}{l}{\textit{$3\times3$ Token Compression}} \\

MLP
& 74.74
& 50.89
& 64.31
& 44.44
& 74.68
& 55.90
& 35.16
& 63.60
& 26.00
& 33.00
& 52.27 \\

\rowcolor{Gray}
TTR-Compressor
& \underline{76.46}
& \textbf{52.40}
& \underline{65.62}
& \underline{45.89}
& \textbf{75.23}
& \underline{58.70}
& \underline{38.27}
& \textbf{67.16}
& \underline{26.80}
& \textbf{36.50}
& \underline{54.30} \\

\rowcolor{Gray}
TTR-Compressor (w/ Conv)
& \textbf{76.63}
& \underline{51.16}
& \textbf{65.88}
& \textbf{47.44}
& \underline{75.16}
& \textbf{60.00}
& \textbf{39.70}
& \textbf{67.16}
& \textbf{28.90}
& \textbf{36.50}
& \textbf{54.85} \\

\end{tabular}
}
\caption{\footnotesize Evaluation across multimodal benchmarks under different token compression ratios. Results are grouped by $2\times2$ and $3\times3$ token compression. Best values within each compression group are \textbf{bolded}, and second-best values are \underline{underlined}.}
\label{tab:mllm_compression}
\vspace{-3mm}
\end{table*}

\textbf{Results.} As shown in Table~\ref{tab:mllm_compression}, TTR-Compressor outperforms the MLP merger under both compression ratios. Under $2\times2$ merging, it raises the average from $54.93$ to $55.92$, and its convolutional variant further improves it to $56.03$. The advantage grows under stronger $3\times3$ merging. The gains are most pronounced on text-rich benchmarks, where fine-grained details are easily lost during merging: under $3\times3$ merging, the convolutional variant improves over the MLP merger by $4.10$ on OCRBench, $4.54$ on InfoVQA, and $3.56$ on ChartQA. On average, it outperforms the MLP merger by $2.58$ points. Notably, TTR-Compressor (w/ Conv) with $3\times3$ merging achieves an average of $54.85$, close to the $54.93$ of the MLP merger with $2\times2$ merging, while using fewer than half as many visual tokens.
We attribute the strong performance of TTR-Compressor to its second-order modeling. The MLP merger concatenates the features in each group and maps them with fixed weights, so the mixing does not depend on how the features relate to each other. When most features in a group are redundant, such as background patches, they dominate the merged token and wash out fine-grained details. In contrast, TTR-Compressor fits correlated features jointly and reduces the influence of redundant ones, so the merged token keeps the distinctive content within the group.

\vspace{-.5mm}
\subsection{Efficiency Analysis}
\vspace{-.5mm}

\begin{wraptable}{r}{0.42\textwidth}
\vspace{-16pt}
\centering
\scriptsize
\renewcommand{\arraystretch}{1.23}
\setlength{\tabcolsep}{7pt}
\setlength{\abovecaptionskip}{5pt}
\setlength{\belowcaptionskip}{3pt}
\begin{tabular}{lcc}
\textbf{Paradigm}
& \textbf{Complexity}
& \textbf{Params} \\
\shline

Mean Pooling
& $\mathcal{O}(Td)$
& $0$ \\

Linear Projection
& $\mathcal{O}(Td^2)$
& $Td^2$ \\

Attention (Reusing)
& $\mathcal{O}(Td)$
& $d$ \\

Attention (Standalone)
& $\mathcal{O}(Td^2)$
& $\sim 4d^2$ \\

\rowcolor{Gray}
TTR-Compressor
& $\mathcal{O}(T^2d)$
& $2d$ \\

\end{tabular}
\vspace{-0.25mm}
\caption{\footnotesize Complexity comparison of different token compression methods ($T \ll d$).}
\label{tab:complexity_comparison}
\vspace{-16pt}
\end{wraptable}
\textbf{Theoretical Complexity and Parameter Efficiency.}
For a sequence $X \in \mathbb{R}^{T \times d}$ compressed into $h \in \mathbb{R}^d$, Table~\ref{tab:complexity_comparison} summarizes the asymptotic complexity and parameter count across different paradigms. Under the targeted regime where $T \ll d$ (e.g., $T \in \{4, 9, 64\}$ vs. $d \ge 1024$), TTR-Compressor fundamentally avoids the $\mathcal{O}(d^2)$ overhead and delivers strong, data-dependent expressiveness with only $2d$ parameters.

\begin{wrapfigure}{r}{0.56\textwidth}
  \vspace{-12pt} 
  \centering
  \includegraphics[width=0.56\textwidth]{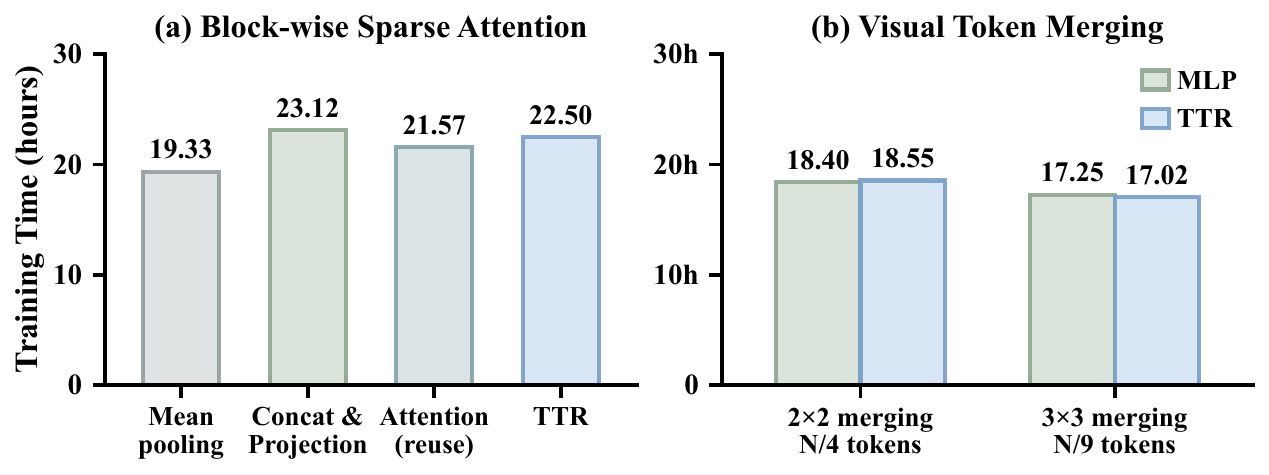}
  \vspace{-20pt} 
  \caption{\footnotesize Wall-clock training time comparisons.}
  \label{fig:wall_clock_time}
  \vspace{-7pt} 
\end{wrapfigure}
\textbf{Empirical Training Efficiency.}
Following the setups in Sections~\ref{sec:exp_sparse} and \ref{sec:exp_visual}, Figure~\ref{fig:wall_clock_time} benchmarks the actual wall-clock training efficiency. For block-wise sparse attention ($8\times\text{H100}$ GPUs), while mean pooling is the most efficient, its representational capacity is inherently limited; other compression variants exhibit comparable training times. For visual token merging ($2\times 8\times\text{H100}$ GPUs), TTR-Compressor achieves efficiency on par with the widely used MLP. Notably, while marginally slower under $2\times 2$ merging, TTR-Compressor outpaces MLP under $3\times 3$ merging, highlighting its superior efficiency and scalability at higher compression ratios.

\vspace{-.5mm}
\subsection{Ablation Study and Analysis}
\vspace{-.5mm}
\begin{table*}[t]
\centering
\scriptsize
\renewcommand{\arraystretch}{1.27}
\setlength{\tabcolsep}{3pt}
\setlength{\abovecaptionskip}{4pt}
\setlength{\belowcaptionskip}{6pt}
\vspace{-1mm}
\resizebox{\linewidth}{!}{%
\begin{tabular}{lccccccccccccccccc}
\multirow{2}{*}{\textbf{Model}}
& \multicolumn{4}{c}{\textbf{S-NIAH-1}}
& \multicolumn{4}{c}{\textbf{S-NIAH-2}}
& \multicolumn{4}{c}{\textbf{S-NIAH-3}}
& \multicolumn{4}{c}{\textbf{MK-NIAH-1}}
& \multirow{2}{*}{\textbf{Avg.}} \\
& 4K & 8K & 16K & 32K
& 4K & 8K & 16K & 32K
& 4K & 8K & 16K & 32K
& 4K & 8K & 16K & 32K
& \\
\shline

TTR-Compressor
& \textbf{0.996}
& \textbf{0.996}
& \textbf{0.992}
& \textbf{0.650}
& \textbf{0.982}
& \textbf{0.566}
& \textbf{0.288}
& \textbf{0.126}
& \textbf{0.770}
& 0.376
& \textbf{0.188}
& 0.066
& \textbf{0.528}
& \textbf{0.302}
& \textbf{0.212}
& 0.110
& \textbf{0.5092} \\

\quad w/o normalize $\mathbf{k}$
& 0.010
& 0.000
& 0.002
& 0.000
& 0.938
& 0.254
& 0.136
& 0.104
& 0.764
& \textbf{0.592}
& 0.186
& \textbf{0.092}
& 0.500
& 0.258
& 0.098
& \textbf{0.116}
& 0.2531 \\

\end{tabular}
}
\caption{\footnotesize Ablation study on key normalization in TTR-Compressor. We compare the full model with a variant without normalizing $\mathbf{k}$ (w/o normalize $\mathbf{k}$). Best values for each context length are \textbf{bolded}.}
\label{tab:ablate_k_norm}
\vspace{-1mm}
\end{table*}

\textbf{Criticality of Key Normalization.}
As shown in Table~\ref{tab:ablate_k_norm}, we find that omitting key normalization ($\mathbf{k}_t = \mathbf{x}_t / \|\mathbf{x}_t\|_2$) precipitates a collapse across retrieval tasks. We attribute this to two reasons.
First, from an optimization perspective, since scalar targets are linearly generated via $v_t = \mathbf{a}^\top \mathbf{x}_t$, using raw keys ($\mathbf{k}_t = \mathbf{x}_t$) makes the static weight $\mathbf{a}$ a trivial exact fit ($\mathbf{k}_t^\top \mathbf{a} \equiv v_t$). 
Consequently, the regularized solution $\mathbf{h}^\star$ is biased toward the fixed vector $\mathbf{a}$, which weakens the instance adaptivity of test-time regression and pushes it toward an input-invariant projection.
Second, from a geometric and numerical viewpoint, outlier tokens with large norms dominate the Gram matrix $\mathbf{K}\mathbf{K}^\top$, causing correlation modeling to reflect token norms rather than semantic similarity.
Normalization resolves both issues at once: it breaks the trivial alignment with static parameters and prevents large-norm tokens from dominating the Gram matrix.

\begin{table*}[t]
\centering
\scriptsize
\renewcommand{\arraystretch}{1.27}
\setlength{\tabcolsep}{3pt}
\setlength{\abovecaptionskip}{5pt}
\setlength{\belowcaptionskip}{3pt}
\resizebox{\linewidth}{!}{%
\begin{tabular}{lccccccccccccccccc}
\multirow{2}{*}{\textbf{Model}}
& \multicolumn{4}{c}{\textbf{S-NIAH-1}}
& \multicolumn{4}{c}{\textbf{S-NIAH-2}}
& \multicolumn{4}{c}{\textbf{S-NIAH-3}}
& \multicolumn{4}{c}{\textbf{MK-NIAH-1}}
& \multirow{2}{*}{\textbf{Avg.}} \\
& 4K & 8K & 16K & 32K
& 4K & 8K & 16K & 32K
& 4K & 8K & 16K & 32K
& 4K & 8K & 16K & 32K
& \\
\shline

$\lambda=0.0001$
& 0.986
& 0.896
& 0.826
& 0.628
& 0.888
& 0.348
& 0.134
& 0.054
& 0.654
& \underline{0.528}
& 0.086
& \underline{0.108}
& 0.384
& 0.242
& 0.118
& \underline{0.102}
& 0.4364 \\

$\lambda=0.001$
& \underline{0.998}
& \underline{0.996}
& 0.990
& 0.948
& \textbf{0.998}
& 0.448
& \underline{0.186}
& \underline{0.112}
& 0.686
& 0.420
& 0.186
& 0.072
& 0.466
& 0.274
& \underline{0.174}
& 0.082
& \underline{0.5023} \\

$\lambda=0.01$
& \textbf{1.000}
& \textbf{1.000}
& \textbf{1.000}
& \textbf{1.000}
& 0.936
& \underline{0.486}
& 0.148
& 0.076
& 0.692
& 0.374
& 0.042
& 0.022
& 0.484
& \underline{0.288}
& 0.146
& 0.068
& 0.4851 \\

$\lambda=0.1$
& 0.990
& 0.986
& 0.982
& \underline{0.978}
& 0.932
& 0.312
& 0.054
& 0.068
& 0.728
& 0.424
& 0.088
& 0.078
& 0.470
& 0.252
& 0.112
& 0.074
& 0.4705 \\

$\lambda=1$
& 0.872
& 0.552
& 0.826
& 0.794
& 0.896
& 0.204
& 0.062
& 0.072
& \underline{0.772}
& 0.496
& 0.184
& 0.088
& 0.526
& 0.278
& 0.090
& 0.068
& 0.4238 \\

Learnable Scalar
& 0.988
& 0.986
& 0.970
& 0.806
& 0.938
& 0.368
& 0.104
& 0.104
& \textbf{0.788}
& \textbf{0.536}
& \textbf{0.212}
& \textbf{0.112}
& \textbf{0.554}
& 0.274
& 0.132
& 0.082
& 0.4971 \\

Learnable Diagonal
& 0.996
& \underline{0.996}
& \underline{0.992}
& 0.650
& \underline{0.982}
& \textbf{0.566}
& \textbf{0.288}
& \textbf{0.126}
& 0.770
& 0.376
& \underline{0.188}
& 0.066
& \underline{0.528}
& \textbf{0.302}
& \textbf{0.212}
& \textbf{0.110}
& \textbf{0.5092} \\

\end{tabular}
}
\caption{\footnotesize Ablation study on the regularization parameter $\lambda$ in TTR-Compressor. Best values for each context length are \textbf{bolded}, and second-best values are \underline{underlined}.}
\label{tab:ablate_regularization}
\vspace{-3mm}
\end{table*}

\textbf{From Static Heuristics to Anisotropic Regularization.}
Table~\ref{tab:ablate_regularization} compares fixed regularization scalars against learnable formulations. All fixed settings outperform the baseline. Moderate values around $\lambda=10^{-2}$ perform well, while performance decreases when $\lambda$ is too small ($0.4364$ at $\lambda=10^{-4}$) or too large ($0.4238$ at $\lambda=1$).
Notably, a \emph{Learnable Scalar} already achieves a competitive average accuracy of $0.4971$, demonstrating that the outer-loop task objective can reliably discover effective regularization strengths end-to-end.
Extending this to a coordinate-wise \emph{Learnable Diagonal} ($\mathbf{D} = \mathrm{diag}(\mathrm{softplus}(\boldsymbol{\rho}))$) achieves the overall best performance. This progression shows the benefit of learning a separate regularization strength for each channel: the model penalizes less useful feature directions more strongly while preserving task-critical ones.

\vspace{-1mm}
\section{Concluding Remarks}
\vspace{-1mm}

In this work, we revisit sequence compression, a pervasive yet often overlooked bottleneck in modern architectures, through the lens of test-time regression. Instead of relying on fixed pooling or learned projections, we formulate token compression as an instance-adaptive ridge regression problem, allowing the compressed representation to be solved directly from each input sequence. 
By exploiting inter-token correlations through second-order interactions, TTR-Compressor offers a more expressive way to preserve sequence information while remaining lightweight and parameter-efficient. Its dual-form solution further makes this formulation scalable in practice.
Empirically, TTR-Compressor performs strongly across both block-wise sparse attention and multimodal token reduction, improving long-context retrieval while matching standard visual mergers with significantly fewer tokens.
More broadly, our results suggest that sequence compression need not be viewed as static aggregation, but as an input-dependent inference problem that adapts to the structure of each sequence. We expect this perspective establishes test-time regression as a simple, expressive, and scalable primitive for efficient foundation models.

\newpage

\bibliography{iclr2027_conference}
\bibliographystyle{iclr2027_conference}

\newpage
\appendix
\onecolumn
\addcontentsline{toc}{section}{Appendix}
\renewcommand \thepart{}
\renewcommand \partname{}
\part{\Large{\centerline{Appendix}}}
\parttoc

\newpage

\section{Derivation of the Dual-form Solutions}
\label{app:dual_proofs}

We first prove the equivalence between the primal and dual forms of ridge regression (Eq.~\ref{eq:ttr_dual}). We then show that, with learnable diagonal regularization, rescaling the keys yields the dual solution in Eq.~\ref{eq:ttr_final}.

\textbf{Equivalence of the Primal and Dual Forms.}

\begin{lemma}[Push-through identity]
\label{lem:push_through}
For any $\bm{K}\in\mathbb{R}^{T\times d}$ and $\lambda>0$,
\begin{equation}
    (\bm{K}^\top\bm{K}+\lambda\bm{I}_d)^{-1}\bm{K}^\top
    =
    \bm{K}^\top(\bm{K}\bm{K}^\top+\lambda\bm{I}_T)^{-1}.
    \label{eq:push_through}
\end{equation}
\end{lemma}

\begin{proof}
Both $\bm{K}^\top\bm{K}$ and $\bm{K}\bm{K}^\top$ are positive semidefinite. Hence, for $\lambda>0$, the matrices $\bm{K}^\top\bm{K}+\lambda\bm{I}_d$ and $\bm{K}\bm{K}^\top+\lambda\bm{I}_T$ are symmetric positive definite and therefore invertible. Expanding both sides directly gives
\begin{equation}
    (\bm{K}^\top\bm{K}+\lambda\bm{I}_d)\,\bm{K}^\top
    =
    \bm{K}^\top\bm{K}\bm{K}^\top+\lambda\bm{K}^\top
    =
    \bm{K}^\top(\bm{K}\bm{K}^\top+\lambda\bm{I}_T).
\end{equation}
Left-multiplying by $(\bm{K}^\top\bm{K}+\lambda\bm{I}_d)^{-1}$ and right-multiplying by $(\bm{K}\bm{K}^\top+\lambda\bm{I}_T)^{-1}$ yields Eq.~\ref{eq:push_through}.
\end{proof}

Right-multiplying Eq.~\ref{eq:push_through} by $\bm{v}$ turns the primal solution in Eq.~\ref{eq:ttr_primal} into the dual form in Eq.~\ref{eq:ttr_dual}:
\begin{equation}
    \bm{h}^{\star}
    =
    (\bm{K}^\top\bm{K}+\lambda\bm{I}_d)^{-1}\bm{K}^\top\bm{v}
    =
    \bm{K}^\top(\bm{K}\bm{K}^\top+\lambda\bm{I}_T)^{-1}\bm{v}.
\end{equation}
Therefore, $\bm{h}^{\star}$ lies in the row space of $\bm{K}$. Its dual coefficients $\bm{\alpha}=(\bm{K}\bm{K}^\top+\lambda\bm{I}_T)^{-1}\bm{v}\in\mathbb{R}^{T}$ are obtained by solving a single $T\times T$ system.

\textbf{Dual Solution under Diagonal Regularization.}

\begin{proposition}
\label{prop:diag_dual}
Let $\bm{D}=\operatorname{diag}(\delta_1,\ldots,\delta_d)$ with $\delta_j>0$, and let $\widetilde{\bm{K}}=\bm{K}\bm{D}^{-1/2}$. Then the solution of Eq.~\ref{eq:ttr_diagonal_objective} satisfies
\begin{equation}
    (\bm{K}^\top\bm{K}+\bm{D})^{-1}\bm{K}^\top\bm{v}
    =
    \bm{D}^{-1/2}\widetilde{\bm{K}}^\top
    \big(\bm{I}_T+\widetilde{\bm{K}}\widetilde{\bm{K}}^\top\big)^{-1}\bm{v}.
    \label{eq:diag_dual}
\end{equation}
\end{proposition}

\begin{proof}
Since $\bm{D}$ is diagonal with positive entries, $\bm{D}^{\pm 1/2}$ are well defined and $\bm{D}^{1/2}\bm{D}^{-1/2}=\bm{I}_d$. Using $\bm{K}=\widetilde{\bm{K}}\bm{D}^{1/2}$, we factorize
\begin{equation}
    \bm{K}^\top\bm{K}+\bm{D}
    =
    \bm{D}^{1/2}\big(\widetilde{\bm{K}}^\top\widetilde{\bm{K}}+\bm{I}_d\big)\bm{D}^{1/2}.
\end{equation}
Inverting this product and using $\bm{K}^\top=\bm{D}^{1/2}\widetilde{\bm{K}}^\top$ gives
\begin{equation}
    (\bm{K}^\top\bm{K}+\bm{D})^{-1}\bm{K}^\top\bm{v}
    =
    \bm{D}^{-1/2}\big(\widetilde{\bm{K}}^\top\widetilde{\bm{K}}+\bm{I}_d\big)^{-1}\bm{D}^{-1/2}\bm{D}^{1/2}\widetilde{\bm{K}}^\top\bm{v}
    =
    \bm{D}^{-1/2}\big(\widetilde{\bm{K}}^\top\widetilde{\bm{K}}+\bm{I}_d\big)^{-1}\widetilde{\bm{K}}^\top\bm{v}.
\end{equation}
Applying Lemma~\ref{lem:push_through} with $\bm{K}\leftarrow\widetilde{\bm{K}}$ and $\lambda=1$ gives Eq.~\ref{eq:diag_dual}.
\end{proof}

Equivalently, the change of variables $\bm{g}=\bm{D}^{1/2}\bm{h}$ turns Eq.~\ref{eq:ttr_diagonal_objective} into a standard ridge problem, $\frac{1}{2}\|\widetilde{\bm{K}}\bm{g}-\bm{v}\|_2^2+\frac{1}{2}\|\bm{g}\|_2^2$. Its dual solution is then mapped back through $\bm{h}^{\star}=\bm{D}^{-1/2}\bm{g}^{\star}$.

Two further properties follow:
\begin{itemize}
    \item The matrix $\bm{I}_T+\widetilde{\bm{K}}\widetilde{\bm{K}}^\top$ is symmetric positive definite, so its Cholesky factorization exists and is unique.
    \item Because $\bm{D}$ is diagonal, forming $\widetilde{\bm{K}}$ and applying $\bm{D}^{-1/2}$ each cost only $\mathcal{O}(Td)$ and $\mathcal{O}(d)$. The overall complexity therefore remains $\mathcal{O}(T^2d+T^3)$.
\end{itemize}
Setting $\bm{D}=\lambda\bm{I}_d$ recovers Eq.~\ref{eq:ttr_dual}.
\section{Redundancy-Aware Weighting of TTR-Compressor}
\label{app:redundancy}

We formalize the claim in Section~\ref{sec:ttr_spectral} that redundant tokens have diminishing influence, while information not explained by correlated tokens keeps a larger relative influence. We take isotropic regularization $\lambda>0$ and treat the targets $\bm v$ as arbitrary, so the results hold for any learned projection $\bm a$. The extension to the learnable $\bm D$ is given at the end. Throughout, let $\bm A=\bm K^\top\bm K+\lambda\bm I_d$, and let $\bm\alpha=(\bm G+\lambda\bm I_T)^{-1}\bm v$, so that $\bm h^\star=\bm K^\top\bm\alpha$.

\textbf{Effective Weights as Unexplained Residuals.}

\begin{proposition}[Residual form]
\label{prop:residual}
For any keys $\bm K$ and targets $\bm v$, the effective weights satisfy
\begin{equation}
    \bm\alpha=\frac{1}{\lambda}\big(\bm v-\bm K\bm h^\star\big).
    \label{eq:residual_form}
\end{equation}
\end{proposition}

\begin{proof}
From $(\bm G+\lambda\bm I_T)\bm\alpha=\bm v$ we have $\lambda\bm\alpha=\bm v-\bm K\bm K^\top\bm\alpha=\bm v-\bm K\bm h^\star$.
\end{proof}

Thus, a token is weighted by the part of its target that the jointly fitted representation leaves unexplained. By contrast, additive pooling assigns weights $\bm v$ regardless of how the other tokens are arranged. To separate the role of the \emph{other} tokens, let $\bm K_{-t}$ and $\bm v_{-t}$ denote the keys and targets with token $t$ removed. Define
\begin{equation}
    \bm A_{-t}=\bm K_{-t}^\top\bm K_{-t}+\lambda\bm I_d,
    \qquad
    \bm h_{-t}=\bm A_{-t}^{-1}\bm K_{-t}^\top\bm v_{-t},
\end{equation}
where $\bm h_{-t}$ is the solution fitted without token $t$.

\begin{proposition}[Leave-one-out form]
\label{prop:loo}
Let $r_t=v_t-\bm k_t^\top\bm h_{-t}$ be the residual of token $t$ under the other tokens, and let $q_t=\bm k_t^\top\bm A_{-t}^{-1}\bm k_t\ge 0$. Then
\begin{equation}
    \alpha_t=\frac{r_t}{\lambda(1+q_t)},
    \qquad
    \bm h^\star=\bm h_{-t}+\frac{r_t}{1+q_t}\,\bm A_{-t}^{-1}\bm k_t.
    \label{eq:loo}
\end{equation}
\end{proposition}

\begin{proof}
Write $\bm A=\bm A_{-t}+\bm k_t\bm k_t^\top$ and $\bm K^\top\bm v=\bm K_{-t}^\top\bm v_{-t}+v_t\bm k_t$. Substituting into $\bm A\bm h^\star=\bm K^\top\bm v$ gives
\begin{equation}
    \bm h^\star=\bm h_{-t}+\big(v_t-\bm k_t^\top\bm h^\star\big)\bm A_{-t}^{-1}\bm k_t.
\end{equation}
Taking the inner product with $\bm k_t$ gives $\bm k_t^\top\bm h^\star=\bm k_t^\top\bm h_{-t}+(v_t-\bm k_t^\top\bm h^\star)\,q_t$, and hence $v_t-\bm k_t^\top\bm h^\star=r_t/(1+q_t)$. Substituting this into the display above and into Eq.~\ref{eq:residual_form} yields Eq.~\ref{eq:loo}.
\end{proof}

Proposition~\ref{prop:loo} shows that token $t$ changes the representation only through $r_t$, the part of its target that the remaining tokens cannot predict. In particular, if the other tokens fully explain token $t$ ($r_t=0$), then $\alpha_t=0$ and $\bm h^\star=\bm h_{-t}$, so the token has no influence. Additive pooling instead always adds $v_t\bm k_t$.

\textbf{Saturation under Repeated Evidence.}
We now quantify this effect for groups of correlated tokens, generalizing the example in Figure~\ref{fig:method}(b).

\begin{proposition}[Group saturation]
\label{prop:saturation}
Suppose the tokens form $C$ groups. Group $c$ contains $m_c$ tokens that share the key $\bm u_c$, with $\|\bm u_c\|_2=1$ and $\bm u_c^\top\bm u_{c'}=0$ for $c\neq c'$. Let $S_c$ be the sum of the targets in group $c$. Then
\begin{equation}
    \bm h^\star=\sum_{c=1}^{C}\frac{S_c}{\lambda+m_c}\,\bm u_c,
    \qquad\text{whereas}\qquad
    \bm K^\top\bm v=\sum_{c=1}^{C}S_c\,\bm u_c .
    \label{eq:saturation}
\end{equation}
\end{proposition}

\begin{proof}
Under this assumption $\bm G$ is block diagonal, and the block for group $c$ is $\bm J_{m_c}=\bm 1\bm 1^\top$. Since $\bm J_m^2=m\bm J_m$, one can verify that $(\bm J_m+\lambda\bm I_m)^{-1}=\lambda^{-1}\big(\bm I_m-\bm J_m/(\lambda+m)\big)$. Hence the weights $\bm\alpha_c$ of group $c$ sum to
\begin{equation}
    \bm 1^\top\bm\alpha_c=\frac{1}{\lambda}\Big(S_c-\frac{m_cS_c}{\lambda+m_c}\Big)=\frac{S_c}{\lambda+m_c}.
\end{equation}
Therefore $\bm h^\star=\sum_c\bm u_c\,\bm 1^\top\bm\alpha_c$, which gives Eq.~\ref{eq:saturation}.
\end{proof}

\begin{corollary}[Diminishing influence of redundancy]
\label{cor:diminishing}
Assume every token has the same target $v>0$, so that $S_c=m_cv$. Compare a group of $m$ repeated tokens with a single distinct token. Their relative influence is
\begin{equation}
    \frac{m/(\lambda+m)}{1/(\lambda+1)}=\frac{m(\lambda+1)}{\lambda+m}<m
    \quad\text{for } m>1.
\end{equation}
This ratio is bounded by $\lambda+1$ for every $m$, whereas the corresponding ratio for additive pooling is $m$. Moreover, the $n$-th repeated token increases the group's coefficient by
\begin{equation}
    \frac{nv}{\lambda+n}-\frac{(n-1)v}{\lambda+n-1}=\frac{\lambda v}{(\lambda+n)(\lambda+n-1)}=\mathcal{O}(n^{-2}),
\end{equation}
which decays with $n$, so repeated evidence has diminishing returns.
\end{corollary}

The ridge strength controls the transition between two regimes:
\begin{itemize}
    \item As $\lambda\to\infty$, $\lambda\bm h^\star\to\bm K^\top\bm v$, which is additive pooling.
    \item As $\lambda\to 0$, each group contributes $S_c/m_c$, its mean target, regardless of how many tokens it contains.
\end{itemize}
Setting $\lambda=1$, $m=3$ recovers Figure~\ref{fig:method}(b): the ratio is $3{:}2$ and $\bm\alpha=(\tfrac14,\tfrac14,\tfrac14,\tfrac12)^\top$.

\textbf{Extension to Learnable Diagonal Regularization.}
All results carry over to learnable diagonal regularization. Using the change of variables in Proposition~\ref{prop:diag_dual}, set $\bm\alpha=(\bm I_T+\widetilde{\bm K}\widetilde{\bm K}^\top)^{-1}\bm v$, so that $\bm h^\star=\bm D^{-1}\bm K^\top\bm\alpha$. Since $\bm K\bm h^\star=\widetilde{\bm K}\widetilde{\bm K}^\top\bm\alpha$, the argument of Proposition~\ref{prop:residual} gives $\bm\alpha=\bm v-\bm K\bm h^\star$. Proposition~\ref{prop:loo} then holds with $\lambda=1$ and $\bm A_{-t}=\bm K_{-t}^\top\bm K_{-t}+\bm D$. Proposition~\ref{prop:saturation} holds when the groups are orthogonal under the learned metric ($\bm u_c^\top\bm D^{-1}\bm u_{c'}=0$ for $c\neq c'$), with $\lambda$ replaced by the group-specific $\lambda_c=(\bm u_c^\top\bm D^{-1}\bm u_c)^{-1}$ and $\bm u_c$ replaced by $\lambda_c\bm D^{-1}\bm u_c$. The learned $\bm D$ therefore decides which feature directions count as redundant, while the diminishing-influence mechanism itself is unchanged.

\newpage
\section{Implementation Details}
\label{sec:implementation_details}

\textbf{Numerical Precision.}
Convolution and target projection use the input precision
(BF16), while regression keys and targets
are cast to FP32 before the dual solve.
Diagonal regularization is evaluated in FP32.
We solve the positive-definite dual system using Cholesky
factorization without explicitly forming an inverse, and cast
the compressed representation back to the input dtype.

\textbf{Local Target Construction.}
Only the regression targets incorporate convolutional context;
keys are computed from the original inputs.
For language, we apply a non-causal depthwise 1D convolution
with kernel size $5$, independently within each compression block,
using two-sided zero padding of width $2$.
The convolution has neither bias nor activation and is initialized
to the identity.
For vision, we use a $3\times3$ depthwise convolution.
Each channel has its own spatial kernel, with no cross-channel
mixing before the scalar target projection.

The base compressor has $2d$ effective trainable parameters:
$d$ for the target projection and $d$ for diagonal regularization.
A bias-free depthwise convolution with $s$ kernel elements adds
$sd$ parameters and $\mathcal{O}(sTd)$ computation.
Thus, the language and vision variants require $7d$ and $11d$
parameters, respectively.
Their total computational complexity remains
$\mathcal{O}(T^2d)$ for fixed kernel sizes and $T\ll d$.

\section{Experimental Details}
\label{app:experimental_details}

\textbf{Block-wise Sparse Attention.}
HiLS-Attention~\citep{hils} argues that small models provide a more diagnostic setting for RULER's predominantly ``find-and-copy'' tasks, as larger models may compensate for retrieval inaccuracies through greater information bandwidth.
We therefore follow their setup and use 345M-parameter models to evaluate compression quality.
All models use 16 Transformer layers, a hidden dimension of 1,024, an FFN dimension of 4,096, and grouped-query attention with 16 query heads and 2 key--value heads.
We use rank-64 query calibration.
Training uses AdamW weight decay of $0.01$ and gradient clipping at a norm of $1.0$; normalization parameters, biases, and embeddings are excluded from weight decay.
The learning rate warms up over the first $0.7\%$ of training and remains constant thereafter.

\textbf{Visual Token Merger.}
We use BF16 training with DeepSpeed ZeRO-1, gradient checkpointing, and a maximum sequence length of 16,384 tokens.
The two stages use per-device batch sizes of 4 and 1, with gradient accumulation over 4 and 8 steps, respectively.
Both stages apply a $3\%$ learning-rate warmup followed by cosine decay, with minimum learning rates of $10^{-5}$ and $0$ and weight decay of $0.01$ and $0.02$, respectively.
For TTR-Compressor, the regression systems are solved in FP32.

\section{Additional Experimental Results}
\label{app:more_results}

\subsection{Scaling Up Model Size and Training Budget}
To further verify that the advantage of TTR-Compressor is not merely an artifact of the smaller-scale setting used in the main paper, we scale up both the model size and the training budget. Specifically, we train a 1B-parameter model on 100B tokens and evaluate it on the same RULER subtasks across context lengths ranging from 4K to 32K. For comparison, we also scale up Attention-Weighted Pooling---the strongest baseline among all compression methods evaluated in the main paper's 345M-parameter experiments---to the same 1B-parameter, 100B-token setting. 

As shown in Table~\ref{tab:ruler_1b}, TTR-Compressor consistently outperforms the Attention-Weighted Pooling baseline across almost all context lengths and subtasks, achieving an average accuracy of 0.7996 compared to 0.735. Notably, the performance gap widens as the context length increases, indicating that TTR-Compressor scales more gracefully to longer contexts even when both model size and training budget are substantially increased relative to the main paper's 345M-parameter setting.

\begin{table*}[h]
\centering
\scriptsize
\renewcommand{\arraystretch}{1.25}
\setlength{\tabcolsep}{3pt}
\setlength{\abovecaptionskip}{5pt}
\setlength{\belowcaptionskip}{3pt}
\newcommand{\gr}[1]{\textcolor{gray}{#1}}
\resizebox{\linewidth}{!}{%
\begin{tabular}{lccccccccccccccccc}
\multirow{2}{*}{\textbf{Model}} & \multicolumn{4}{c}{\textbf{S-NIAH-1}} & \multicolumn{4}{c}{\textbf{S-NIAH-2}} & \multicolumn{4}{c}{\textbf{S-NIAH-3}} & \multicolumn{4}{c}{\textbf{MK-NIAH-1}} & \multirow{2}{*}{\textbf{Avg.}}\\
 & 4K & 8K & 16K & 32K & 4K & 8K & 16K & 32K & 4K & 8K & 16K & 32K & 4K & 8K & 16K & 32K &\\
\shline
 Attention-Weighted Pooling
 & 0.978 & 0.972 & 0.984 & 0.990 & \textbf{1.000} & \textbf{0.982} & 0.860 & 0.732 & 0.898 & 0.856 & 0.456 & 0.550 & 0.494 & 0.358 & 0.336 & 0.314 & 0.735\\
\rowcolor{Gray}
 TTR-Compressor
 & \textbf{1.000} & \textbf{1.000} & \textbf{1.000} & \textbf{1.000} & \textbf{1.000} & 0.956 & \textbf{0.862} & \textbf{0.750} & \textbf{0.980} & \textbf{0.918} & \textbf{0.740} & \textbf{0.662} & \textbf{0.626} & \textbf{0.442} & \textbf{0.446} & \textbf{0.412} & \textbf{0.7996}\\
\end{tabular}
}
\caption{\footnotesize Accuracy on Single Needle-In-A-Haystack (S-NIAH) and Multi-Key Needle-In-A-Haystack (MK-NIAH) tasks from RULER across varied context lengths. Results are based on a 1B-parameter model trained on 100B tokens. Best values for each context length are \textbf{bolded}.}
\label{tab:ruler_1b}
\vspace{-1mm}
\end{table*}

\subsection{Ablation on the Value Parameterization}
\label{sec:v_ablation}
Under the $3\times3$ compression setup, we compare five parameterizations of $v_t$: (i) \emph{constant} ($v_t=1$); (ii) \emph{linear} ($v_t=\boldsymbol{a}^{\top}\boldsymbol{x}_t$); (iii) \emph{MLP} ($v_t=\mathrm{MLP}(\boldsymbol{x}_t)$); (iv) \emph{bilinear} ($v_t=\mathrm{MLP}(\boldsymbol{x}_t)^{\top}\boldsymbol{x}_t$); and (v) their convolutional variants using local $3\times3$ features $\mathrm{Conv}(\boldsymbol{X})_t$.

As shown in Table~\ref{tab:v_ablation}, we observe:
\textbf{(1) Content dependence is critical}: The constant target yields the lowest average (53.56), while the linear target gains $+0.74$, showing that regression targets must encode token content to form effective summaries.
\textbf{(2) Deeper heads offer marginal gains}: Plain MLP matches the linear baseline (54.30), and adding convolution degrades it (53.91). Although the bilinear head slightly improves performance (54.55), its convolutional counterpart sees no further gain (54.37).
\textbf{(3) Spatial context outweighs head capacity}: Combining local convolution with a linear readout, $v_t=\boldsymbol{a}^{\top}\mathrm{Conv}(\boldsymbol{X})_t$, achieves the best overall score (54.85).
These results indicate that value capacity is not the bottleneck; local neighborhood aggregation provides a more effective inductive bias than non-linear projection. 

\begin{table*}[t]
\centering
\scriptsize
\renewcommand{\arraystretch}{1.23}
\setlength{\tabcolsep}{3pt}
\setlength{\abovecaptionskip}{5pt}
\setlength{\belowcaptionskip}{3pt}
\resizebox{\linewidth}{!}{%
\begin{tabular}{llccccccccccc}
\multirow{2}{*}{\textbf{Method}} & \multirow{2}{*}{\textbf{$\boldsymbol{v}_t$}} & \multicolumn{3}{c}{\textbf{General VQA}} & \multicolumn{2}{c}{\textbf{Knowledge Reasoning}} & \multicolumn{3}{c}{\textbf{Document Understanding}} & \multicolumn{2}{c}{\textbf{Spatial Under.}} & \multirow{2}{*}{\textbf{Avg.}} \\
& & \textbf{MMBench} & \textbf{MMStar} & \textbf{RWQA} & \textbf{MMMU} & \textbf{AI2D} & \textbf{OCRBench} & \textbf{InfoVQA} & \textbf{ChartQA} & \textbf{MMSI} & \textbf{ERQA} & \\
\shline
\multicolumn{13}{l}{\textit{Constant value}} \\
Constant & $1$ & 75.26 & 51.63 & 64.31 & 45.67 & \underline{76.07} & 55.90 & 37.84 & \underline{67.20} & 26.80 & 35.00 & 53.56 \\
\hline
\multicolumn{13}{l}{\textit{Linear value}} \\
\rowcolor{Gray}
TTR-Compressor & $\boldsymbol{a}^{\top}\boldsymbol{x}_t$ & \underline{76.46} & \textbf{52.40} & \underline{65.62} & 45.89 & 75.23 & 58.70 & 38.27 & 67.16 & 26.80 & 36.50 & 54.30 \\
\rowcolor{Gray}
TTR-Compressor (w/ Conv) & $\boldsymbol{a}^{\top}\mathrm{Conv}(\boldsymbol{X})_t$ & \textbf{76.63} & 51.16 & \textbf{65.88} & \textbf{47.44} & 75.16 & \textbf{60.00} & \textbf{39.70} & 67.16 & \underline{28.90} & 36.50 & \textbf{54.85} \\
\hline
\multicolumn{13}{l}{\textit{MLP value}} \\
MLP & $\mathrm{MLP}(\boldsymbol{x}_t)$ & 76.29 & \underline{52.06} & 63.92 & 44.78 & \textbf{76.42} & 58.30 & 38.66 & 67.00 & 28.10 & \underline{37.50} & 54.30 \\
MLP (w/ Conv) & $\mathrm{MLP}(\mathrm{Conv}(\boldsymbol{X})_t)$ & \underline{76.46} & 50.87 & 64.18 & 44.56 & 75.32 & 58.10 & 38.09 & 66.80 & 26.50 & \textbf{38.25} & 53.91 \\
\hline
\multicolumn{13}{l}{\textit{MLP-dot value}} \\
MLP-dot & $\mathrm{MLP}(\boldsymbol{x}_t)^{\top}\boldsymbol{x}_t$ & 76.20 & 51.84 & 65.23 & \underline{47.33} & 75.55 & \underline{59.60} & 38.47 & \textbf{68.88} & 27.40 & 35.00 & \underline{54.55} \\
MLP-dot (w/ Conv) & $\mathrm{MLP}(\mathrm{Conv}(\boldsymbol{X})_t)^{\top}\boldsymbol{x}_t$ & 76.29 & 51.08 & 64.05 & 46.56 & 75.26 & 58.10 & \underline{38.81} & 66.36 & \textbf{29.90} & 37.25 & 54.37 \\
\end{tabular}
}
\caption{\footnotesize Ablation on the parameterization of the value target $\boldsymbol{v}_t$ in the TTR objective $\boldsymbol{h}^{\star}=\arg\min_{\boldsymbol{h}}\frac{1}{2}\|\boldsymbol{K}\boldsymbol{h}-\boldsymbol{v}\|_2^2+\frac{1}{2}\boldsymbol{h}^{\top}\boldsymbol{D}\boldsymbol{h}$. All variants use $3\times3$ token compression. Gray rows denote the TTR-Compressor configurations used in the main paper. Best values are \textbf{bolded}, and second-best values are \underline{underlined}.}
\label{tab:v_ablation}
\vspace{-3mm}
\end{table*}

\section{Inference Latency of Block-wise Sparse Attention}
\label{app:latency}

As shown in Figure~\ref{fig:latency}, all compression methods exhibit similar prefill and decode latency and substantially outperform full attention for context lengths beyond 16K tokens.

\begin{figure*}[h]
    \centering
    \includegraphics[width=0.9\textwidth]{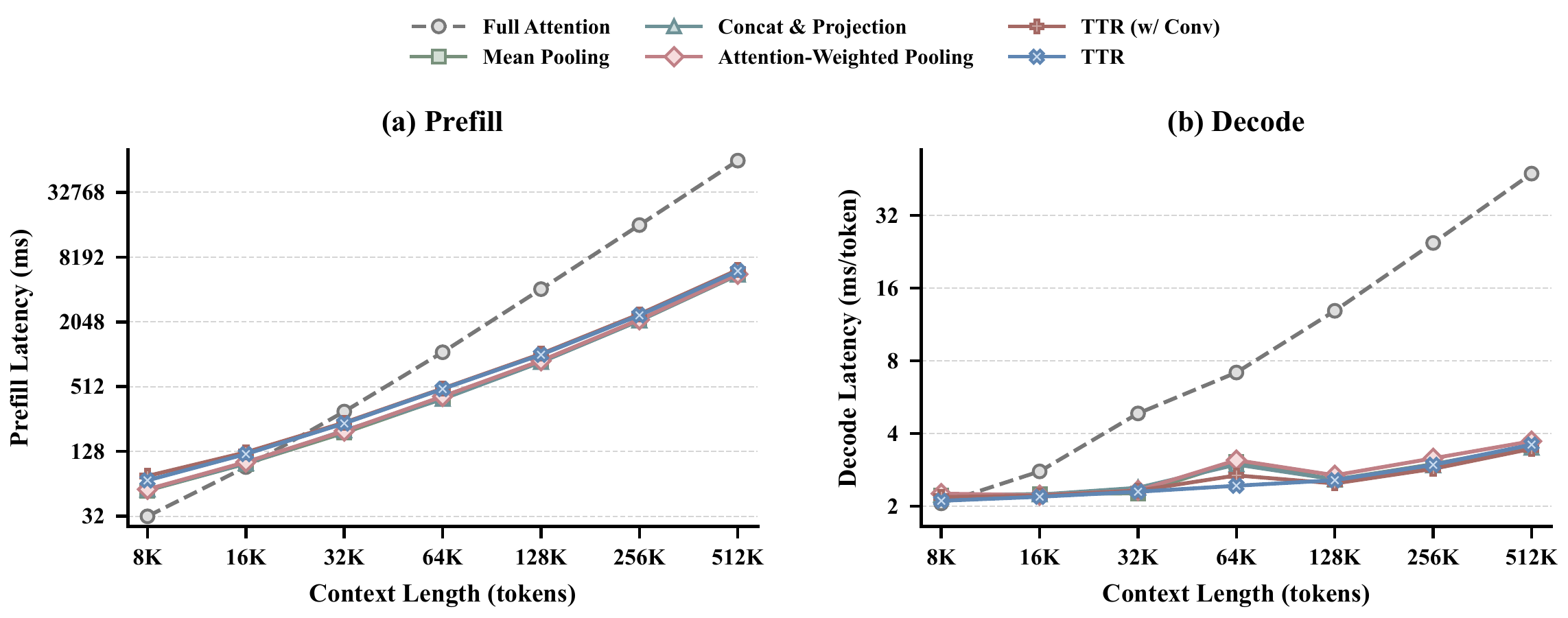}
    \caption{\footnotesize Prefill and decode latency across context lengths.}
    \label{fig:latency}
\end{figure*}

\end{document}